\documentclass{article}

\usepackage{microtype}
\usepackage{graphicx}
\usepackage{booktabs} 

\usepackage{hyperref}
\usepackage{comment}

\usepackage{tikz}
\usetikzlibrary{positioning}
\usetikzlibrary{arrows.meta}

\usepackage{dsfont}
\usepackage{multirow}
\usepackage{caption}
\usepackage{subcaption}
\usepackage{amsmath,amsthm,amssymb,mathtools}

\newtheorem{theorem}{Theorem}
\newtheorem{lemma}[theorem]{Lemma}

\newtheorem{problem}[theorem]{Problem}
\newtheorem{remark}[theorem]{Remark}

\newtheorem{example}[theorem]{Example}
\newtheorem{question}[theorem]{Question}

\makeatletter
\providecommand*{\cupdot}{%
  \mathbin{%
    \mathpalette\@cupdot{}%
  }%
}
\newcommand*{\@cupdot}[2]{%
  \ooalign{%
    $\m@th#1\cup$\cr
    \hidewidth$\m@th#1\cdot$\hidewidth
  }%
}

\usepackage[accepted]{icml2020}

\icmltitlerunning{Graph Homomorphism Convolution}

\begin{document}

\twocolumn[
\icmltitle{Graph Homomorphism Convolution}

\icmlsetsymbol{equal}{*}

\begin{icmlauthorlist}
\icmlauthor{Hoang NT}{riken,titech}
\icmlauthor{Takanori Maehara}{riken}
\end{icmlauthorlist}

\icmlaffiliation{riken}{RIKEN Center for Advanced Intelligence Project, Tokyo, Japan}
\icmlaffiliation{titech}{Tokyo Institute of Technology, Tokyo, Japan}

\icmlcorrespondingauthor{Hoang NT}{me@gearons.org}

\icmlkeywords{Graph Learning, Graph Classification, Graph Homomorphism, Machine Learning, ICML}

\vskip 0.3in
]



\printAffiliationsAndNotice{}  

\begin{abstract}
In this paper, we study the graph classification problem from the graph homomorphism perspective. 
We consider the homomorphisms from $F$ to $G$, where $G$ is a graph of interest (e.g. molecules or social networks) and $F$ belongs to some family of graphs (e.g. paths or non-isomorphic trees). 
We show that graph homomorphism numbers provide a natural invariant (isomorphism invariant and $\mathcal{F}$-invariant) embedding maps which can be used for graph classification.
Viewing the expressive power of a graph classifier by the $\mathcal{F}$-indistinguishable concept, we prove the universality property of graph homomorphism vectors in approximating $\mathcal{F}$-invariant functions.
In practice, by choosing $\mathcal{F}$ whose elements have bounded tree-width, we show that the homomorphism method is efficient compared with other methods.
\end{abstract}

\section{Introduction}\label{sec:intro}

\subsection{Background}

In many fields of science, objects of interest often exhibit irregular structures. 
For example, in biology or chemistry, molecules and protein interactions are often modeled as graphs~\cite{milo2002network,benson2016higher}.
In multi-physics numerical analyses, methods such as the finite element methods discretize the sample under study by 2D/3D-meshes~\cite{mezentsev2004generalized,fey2018spline}. 
In social studies, interactions between people are presented as a social network~\cite{barabasi2016network}. 
Understanding these irregular non-Euclidean structures have yielded valuable scientific and engineering insights.
With recent successful developments of machine learning on regular Euclidean data such as images, a natural extension challenge arises: \emph{How do we learn non-Euclidean data such as graphs or meshes?}

Geometric (deep) learning \cite{bronstein2017geometric} is an important extension of machine learning as it generalizes learning methods from Euclidean data to non-Euclidean data. 
This branch of machine learning not only deals with learning irregular data but also provides a proper means to combine meta-data with their underlying structure. 
Therefore, geometric learning methods have enabled the application of machine learning to real-world problems: From categorizing complex social interactions to generating new chemical molecules.
Among these methods, graph-learning models for the classification task have been the most important subject of study.

Let $\mathcal{X}$ be the space of features (e.g., $\mathcal{X} = \mathbb{R}^d$ for some positive integer $d$), $\mathcal{Y}$ be the space of outcomes (e.g., $\mathcal{Y} = \{ 0, 1 \}$), and $G = (V(G), E(G))$ be a graph with a vertex set $V(G)$ and edge set $E(G) \subseteq V(G) \times V(G)$. The graph classification problem is stated follow\footnote{This setting also includes the regression problem.}.
\begin{problem}[Graph Classification Problem] 
\label{prob:gclf}
\label{prob:graph-learning-problem}
We are given a set of tuples $\{ (G_i, x_i, y_i) : i = 1, \dots, N \}$ of graphs $G_i = (V(G_i), E(G_i))$, vertex features $x_i \colon V(G_i) \to \mathcal{X}$, and outcomes $y_i \in \mathcal{Y}$.
The task is to learn a hypothesis
$h$ such that $h((G_i, x_i)) \approx y_i$.
\footnote{$h$ can be a machine learning model with a given training set.}
\end{problem}

Problem~\ref{prob:gclf} has been studied both theoretically and empirically.
Theoretical graph classification models often discuss the universality properties of some targeted function class.
While we can identify the function classes which these theoretical models can approximate, practical implementations pose many challenges. For instance, the tensorized model proposed by \cite{keriven2019universal} is universal in the space of continuous functions on bounded size graphs, but it is impractical to implement such a model.
On the other hand, little is known about the class of functions which can be estimated by some practical state-of-the-art models.
To address these disadvantages of both theoretical models and practical models, we need a practical graph classification model whose approximation capability can be parameterized.
Such a model is not only effective in practice, as we can introduce inductive bias to the design by the aforementioned parameterization, but also useful in theory as a framework to study the graph classification problem.


In machine learning, a model often introduces a set of assumptions, which is known as inductive bias.
These assumptions help narrow down the hypothesis space while maintaining the validity of the learning model subject to the nature of the data.
For example, a natural inductive bias for graph classification problems is the invariant to the permutation property \cite{maron2018invariant,sannai2019universal}.
We are often interested in a hypothesis $h$ that is invariant to isomorphism, i.e., for two isomorphic graphs $G_1$ and $G_2$ the hypothesis $h$ should produce the same outcome, $h(G_1) = h(G_2)$.
Therefore, it is reasonable to restrict our attention to only invariant hypotheses.
More specifically, we focus on invariant embedding maps because we can construct an invariant hypothesis by combining these mappings with any machine learning model designed for vector data.
Consider the following research question:

\begin{question}
\label{q:main_question}
How to design an efficient and invariant embedding map for the graph classification problem?
\end{question}

\subsection{Homomorphism Numbers as a Classifier}

A common approach to Problem~\ref{prob:graph-learning-problem} is to design an \emph{embedding}\footnote{Not to be confused with ``vertex embedding''.} $\rho \colon (G, x) \mapsto \rho((G, x)) \in \mathbb{R}^p$, which maps graphs to vectors, where $p$ is the dimensionality of the representation.
Such an embedding can be used to represent a hypothesis for graphs as $h((G, x)) = g(\rho((G, x))$ by some hypothesis $g \colon \mathbb{R}^p \to \mathcal{Y}$ for vectors.
Because the learning problem on vectors is a well-studied problem, we can focus on designing and understanding graph embedding.

We found that using homomorphism numbers as an invariant embedding is not only theoretically valid but also extremely efficient in practice.
In a nutshell, the embedding for a graph $G$ is given by selecting $k$ pattern graphs to form a fixed set $\mathcal{F}$, then computing the homomorphism numbers from each $F \in \mathcal{F}$ to $G$. 
The classification capability of the homomorphism embedding is parameterized by $\mathcal{F}$.
We develop rigorous analyses for this idea in Section~\ref{sec:graph_no_features} (without vertex features) and Section~\ref{sec:graph_with_features} (with vertex features).

Our contribution is summarized as follows:
\vspace{-1em}
\begin{itemize}
    \setlength\itemsep{0em}
    \item Introduce and analyze the usage of weighted graph homomorphism numbers with a general choice of $\mathcal{F}$. The choice of $\mathcal{F}$ is a novel way to parameterize the capability of graph learning models compared to choosing the tensorization order in other related work.
    \item Prove the universality of the homomorphism vector in approximating $\mathcal{F}$-indistinguishable functions. Our main proof technique is to check the condition of the Stone-Weierstrass theorem. 
    \item Empirically demonstrate our theoretical findings with synthetic and benchmark datasets. Notably, we show that our methods perform well in graph isomorphism test compared to other machine learning models. 
\end{itemize}

In this paper, we focus on simple undirected graphs without edge weights for simplicity. The extension of all our results to directed and/or weighted graphs is left as future work.

\subsection{Related Works}

There are two main approaches to construct an embedding: \emph{graph kernels} and \emph{graph neural networks}.
In the following paragraphs, we introduce some of the most popular methods which directly related to our work.
For a more comprehensive view of the literature, we refer to surveys on graph neural networks~\cite{wu2019comprehensive} and graph kernels~\cite{gartner2003survey,kriege2019survey}.

\subsubsection{Graph Kernels}

The kernel method first defines a kernel function on the space, which implicitly defines an embedding $\rho$ such that the inner product of the embedding vectors gives a kernel function.
Graph kernels implement $\rho$ by counting methods or graph distances (often exchangeable measures).
Therefore, they are isomorphism-invariant by definition.

The graph kernel method is the most popular approach to study graph embedding maps. Since designing a kernel which uniquely represents graphs up to isomorphisms is as hard as solving graph isomorphism \cite{gartner2003graph}, many previous studies on graph kernels have focused on proposing a solution to the trade-off between computational efficiency and representability. A natural idea is to compute subgraph frequencies \cite{gartner2003graph} to use as graph embeddings. However, counting subgraphs is a \#W[1]-hard problem~\cite{flum2006parameterized} and even counting induced subgraphs is an NP-hard problem (more precisely it is an \#A[1]-hard problem~\cite{flum2006parameterized}).
Therefore, methods like the tree kernel \cite{collins2002convolution,mahe2009graph} or the random walk kernel \cite{gartner2003graph,borgwardt2005protein} restrict the subgraph family to be some computationally efficient graphs. 
Regarding graph homomorphism, \citeauthor{gartner2003graph} and also \citeauthor{mahe2009graph} studied a relaxation which is similar to homomorphism counting (walks and trees). 
Especially, \citeauthor{mahe2009graph} showed that the tree kernel is efficient for molecule applications.  
However, their studies limit to tree kernels and it is not known to what extend these kernels can represent graphs. 

More recently, the graphlet kernel \cite{shervashidze2009efficient,prvzulj2004modeling} and the Weisfeiler-Lehman kernel \cite{shervashidze2011weisfeiler,kriege2016valid} set the state-of-the-art for benchmark datasets~\cite{tud}. 
Other similar kernels with novel modifications to the distance function, such as Wasserstein distance, are also proposed~\cite{togninalli2019wasserstein}. 
While these kernels are effective for benchmark datasets, some are known to be not universal~\cite{gin,keriven2019universal} and it is difficult to address their expressive power to represent graphs.

\subsubsection{Graph Neural Networks}

\begin{figure}
\begin{center}
   \includegraphics[width=0.9\linewidth]{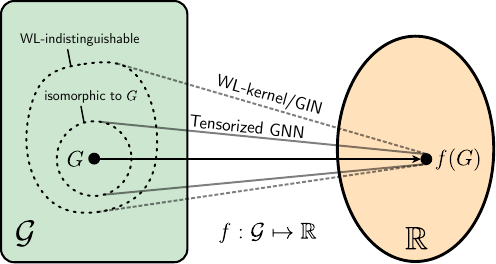}
\end{center}
   \caption{A visualization of Graph Neural Networks' expressive power.
   An ``ideal'' GNN, for instance the tensorized GNN by
   \citeauthor{keriven2019universal}, maps graphs that are isomorphic to $G$ to $f(G)$.
   In contrast, WL-kernel~\cite{shervashidze2011weisfeiler} and (ideal) GIN~\cite{gin} is limited by the WL-indistinguishable set so they might map graphs which are non-isomorphic to $G$ to $f(G)$.}
\label{fig:gfunc}
\vspace{-1em}
\end{figure}

Graph Neural Networks refers to a new class of graph classification models in which the embedding map $\rho$ is implemented by a neural network. In general, the mapping $\rho$ follows an aggregation-readout scheme~\cite{graphsage,gilmer2017neural,gin,gntk} where vertex features are aggregated from their neighbors and then read-out to obtain the graph embedding. Empirically, especially on social network datasets, these neural networks have shown better accuracy and inference time than graph kernels. However, there exist some challenging cases where these practical neural networks fail, such as Circular Skip Links synthetic data~\cite{murphy19pooling} or bipartite classification (Section~\ref{sec:exp}).

Theoretical analysis of graph neural networks is an active topic of study. The capability of a graph neural network has been recently linked to the Weisfeiler-Lehman isomorphism test~\cite{morris2019weisfeiler,gin}. 
Since \citeauthor{morris2019weisfeiler} and \citeauthor{gin} proved that the aggregation-readout scheme is bounded by the one-dimensional Weisfeiler-Lehman test, much work has been done to quantify and improve the capability of graph neural networks via the tensorization order.
Another important aspect of graph neural networks is their ability to approximate graph isomorphism equivariant or invariant functions~\cite{maron2018invariant,maron2019universality,keriven2019universal}.
Interestingly, \citeauthor{chen2019equivalence} showed that isomorphism testing and function approximation are equivalent.

The advantage of tensorized graph neural networks lies in their expressive power.
However, the disadvantage is that the tensorization order makes it difficult to have an intuitive view of the functions which need to be approximated.
Furthermore, the empirical performance of these models might heavily depends on initialization~\cite{chen2019equivalence}.

Figure~\ref{fig:gfunc} visualizes the interaction between function approximation and isomorphism testing.
An ideal graph neural network $f$ maps only $G$ and graphs isomorphic to it to $f(G)$.
On the other hand, efficient implementation of $f$ can only maps some $\mathcal{F}$-indistinguishable of $G$ to $f(G)$.
This paper shows that graph homomorphism vectors with some polynomials universally approximate $\mathcal{F}$-invariant functions. 





\section{Graphs without Features}
\label{sec:graph_no_features}

We first establish our theoretical framework for graphs without vertex features.
Social networks are often feature-less graphs, in which only structural information (e.g. hyperlinks, friendships, etc.) is captured.
The main result of this section is to show that using homomorphism numbers with some polynomial not only yields a universally invariant approximator but that we can also select the pattern set $\mathcal{F}$ for some targeted applications.

\subsection{Definition}

An (undirected) \emph{graph} $G = (V(G), E(G))$ is \emph{simple} if it has neither self-loops nor parallel edges.
We denote by $\mathcal{G}$ the set of all simple graphs.

Let $G$ be a graph.
For a finite set $U$ and a bijection $\sigma \colon V(G) \to U$, we denote by $G^\sigma$ the graph defined by $V(G^\sigma) = U$ and $E(G^\sigma) = \{ (\sigma(u), \sigma(v)) : (u, v) \in E(G) \}$.
Two graphs $G_1$ and $G_2$ are \emph{isomorphic} if $G_1^\sigma = G_2$ for some bijection $\sigma \colon V(G_1) \to V(G_2)$.

\subsection{Homomorphism Numbers}

Here, we introduce the homomorphism number.
This is a well-studied concept in graph theory~\cite{hell2004graphs,lovasz2012large} and plays a key role in our framework.

Let $F$ and $G$ be undirected graphs.
A \emph{homomorphism from $F$ to $G$} is a function $\pi \colon V(F) \to V(G)$ that preserves the existence of edges, i.e., $(u, v) \in E(F)$ implies $(\pi(u), \pi(v)) \in E(G)$.
We denote by $\mathrm{Hom}(F, G)$ the set of all homomorphisms from $F$ to $G$.
The \emph{homomorphism number} $\mathrm{hom}(F, G)$ is the cardinality of the homomorphisms, i.e., $\mathrm{hom}(F, G) = |\mathrm{Hom}(F, G)|$. 
We also consider the \emph{homomorphism density} $t(F, G)$.
This is a normalized version of the homomorphism number:
\begin{align}
\label{eq:hom-density-graph}
    t(F, G) &= \frac{\mathrm{hom}(F, G)}{|V(G)|^{|V(F)|}} \\
    &= \sum_{\pi: V(F) \to V(G)} \prod_{u \in V(F)} \frac{1}{|V(G)|} \nonumber \\ & \times \prod_{(u,v) \in E(F)} 1[ (\pi(u), \pi(v)) \in E(G) ],
\label{eq:hom-density-graph-2ndline}
\end{align}
where $1[ \cdot ]$ is the Iverson braket. 
Eq.~\eqref{eq:hom-density-graph-2ndline} can be seen as the probability that randomly sampled $|V(F)|$ vertices of $V(G)$ preserve the edges of $E(F)$.
Intuitively, a homomorphism number $\mathrm{hom}(F, G)$ aggregates local connectivity information of $G$ using a pattern graph $F$.

\begin{example}
\label{ex:singleton}
Let $\circ$ be a single vertex, we have $\mathrm{hom}(\circ, G) = |V(G)|$ and $\mathrm{hom}(\begin{tikzpicture}
\node[draw,circle,inner sep=1.3pt] at (0,0) (a) {};
\node[draw,circle,inner sep=1.3pt] at (0.3,0) (b) {};
\draw[-] (a)--(b);\end{tikzpicture}, G) = 2 |V(E)|$. 
\end{example}
\begin{example}
\label{ex:stars}
Let $S_k$ be the star graph of size $k+1$.
Then, $\mathrm{hom}(S_k, G) \propto \sum_{u \in V(G)} d(u)^k$, where $d(u)$ is the degree of vertex $u$.
\end{example}
\begin{example}
\label{ex:cycles}
We have: $\mathrm{hom}(C_k, G) \propto \mathrm{tr}(A^k)$, where $C_k$ is a length $k$ cycle and $A$ is the adjacency matrix of $G$.
\end{example}


It is trivial to see that the homomorphism number is invariant under isomorphism.
Surprisingly, the converse holds as homomorphism numbers identify the isomorphism class of a graph.
Formally, we have the following theorem.
%
\begin{theorem}[\cite{lovasz1967operations}]
\label{thm:lovasz1967operations}
Two graphs $G_1$ and $G_2$ are isomorphic if and only if $\mathrm{hom}(F, G_1) = \mathrm{hom}(F, G_2)$ for all simple graphs $F$.
In addition, if $|V(G_1)|, |V(G)_2| \le n$ then we only have to examine $F$ with $|V(F)| \le n$.
\end{theorem}

\subsection{Homomorphism Numbers as Embeddings}

The isomorphism invariance of the homomorphism numbers motivates us to use them as the embedding vectors for a graph.
Since examining all graphs will be impractical (i.e. $\mathcal{F} = \mathcal{G}$), we select a subset $\mathcal{F} \subseteq \mathcal{G}$ as a parameter for the graph embedding.
We obtain the embedding vector of a graph $G$ by stacking the the homomorphism numbers from $F \in \mathcal{F}$. When $\mathcal{F} = \mathcal{G}$, this is known as the Lov\'{a}sz vector.
$$\mathrm{hom}(\mathcal{F}, G) = [\, \mathrm{hom}(F, G) : F \in \mathcal{F} \,] \in \mathbb{R}^{|\mathcal{F}|}.$$
We focus on two criterion: Expressive capability and computational efficiency.
Similar to the kernel representability and efficiency trade-off,
a more expressive homomorphism embedding map is usually less efficient and vice versa.

Graphs $G_1$ and $G_2$ are defined to be \emph{$\mathcal{F}$-indistinguishable} if $\mathrm{hom}(F, G_1) = \mathrm{hom}(F, G_2)$ for all $F \in \mathcal{F}$~\cite{boker2019complexity}.
Theorem~\ref{thm:lovasz1967operations} implies that the $\mathcal{F}$-indistinguishability generalizes graph isomorphism.
For several classes $\mathcal{F}$, the interpretation of $\mathcal{F}$-indistinguishability is studied; the results are summarized in Table~\ref{tab:F-indistinguishable}.
The most interesting result is the case when $F \in \mathcal{F}$ has treewidth\footnote{Smaller treewidth implies the graph is more ``tree-like''.} at most $k$ where the $\mathcal{F}$-indistinguishability coincides with the $k$-dimensional Weisfeiler--Lehman isomorphism test~\cite{dell2018lovasz}.

\begin{table}[tb]
    \caption{Meaning of $\mathcal{F}$-indistinguishable}
    \label{tab:F-indistinguishable}
    \centering
    \begin{tabular}{p{0.15\textwidth} p{0.28\textwidth}}
    \toprule
         $\mathcal{F}$ & $\mathcal{F}$-indistinguishable \\
    \toprule
        single vertex & graphs have the same number of vertices (Example~\ref{ex:singleton}) \\
    \midrule
        single edge & graphs have the same number of edges (Example~\ref{ex:singleton}) \\
    \midrule
        stars & graphs have the same degree sequence (Example~\ref{ex:stars}) \\ 
    \midrule
         cycles & adjacency matrices have the same eigenvalues (Example~\ref{ex:cycles}) \\
    \midrule
         treewidth up to $k$ &
         graphs cannot be distinguished by the $k$-dimensional Weisfeiler-Lehman test~\cite{dell2018lovasz} \\
    \midrule
         all simple graphs &
         isomorphic graphs~\cite{lovasz1967operations} \\
    \bottomrule
    \end{tabular}
\end{table}

A function $f \colon \mathcal{G} \to \mathbb{R}$ is \emph{$\mathcal{F}$-invariant} if $f(G_1) = f(G_2)$ for all $\mathcal{F}$-indistinguishable $G_1$ and $G_2$;
therefore, if we use the $\mathcal{F}$-homomorphism as an embedding, we can only represent $\mathcal{F}$-invariant functions.
In practice, $\mathcal{F}$ should be chosen as small as possible such that the target hypothesis can be assumed to be $\mathcal{F}$-invariant.
In the next section, we show that any continuous $\mathcal{F}$-invariant function is arbitrary accurately approximated by a function of the $\mathcal{F}$-homomorphism embedding (Theorem~\ref{thm:bounded} and \ref{thm:unbounded}).


\subsection{Expressive Power: Universality Theorem}

By characterizing the class of functions that is represented by $\mathrm{hom}(\mathcal{F}, G)$, we obtain the following results.

\begin{theorem}
\label{thm:bounded}
Let $f$ be an $\mathcal{F}$-invariant function.
For any positive integer $N$, there exists a degree $N$ polynomial $h_N$ of $\mathrm{hom}(\mathcal{F}, G)$ s.t. $f(G) \approx h_N(G) \ \forall$ $G$ with $|V(G)| \le N$.
\end{theorem}
\begin{theorem}
\label{thm:unbounded}
Let $f$ be a continuous $\mathcal{F}$-invariant function.
There exists a degree $N$ polynomial $h_N$ of $\mathrm{hom}(F, G)$ ($F \in \mathcal{F}$) such that $f(G) \approx h_N(G) \ \forall G \in \mathcal{G}$.
\end{theorem}

\vspace{-1em}
\begin{proof}[Proof of Theorem~\ref{thm:bounded}]
Since $|V(G)| \le N$, the graph space contains a finite number of points; therefore, under the discrete topology, the space is compact Hausdorff.\footnote{In this topology, any function is continuous.}
Let $X$ be a set of points (e.g., graphs).
A set of functions $A$ \emph{separates} $X$ if 
for any two different points $G_1, G_2 \in X$, there exists a function $h \in A$ such that $h(G_1) \neq h(G_2)$.
By this separability and the Stone--Weierstrass theorem (Theorem~\ref{thm:stone-weierestrass}), we conclude the proof.
\end{proof}
\vspace{-1em}

Theorem~\ref{thm:bounded} is the universal approximation theorem for bounded size graphs.
This holds without any assumption of the target function $f$.
It is worth mentioning that the invariant/equivariant universality results of \emph{tensorized graph neural networks} on this bounded size setting were proven by~\cite{keriven2019universal}; the unbounded case remains an open problem.
Theorem~\ref{thm:unbounded} is the universal approximation for all graphs (unbounded).
This is an improvement to the previous works.
However, our theorem only holds for continuous functions, where the topology of the space has to satisfy the conditions of the Stone-Weierstrass theorem.


\begin{theorem}[Stone--Weierstrass Theorem~\cite{hart2003encyclopedia}] 
\label{thm:stone-weierestrass}
Let $X$ be a compact Hausdorff space and $C(X)$ be the set of continuous functions from $X$ to $\mathbb{R}$.
If a subset $A \subseteq C(X)$ separates $X$ then the set of polynomials of $A$ is dense in $C(X)$ w.r.t. the topology of uniform convergence.
\end{theorem}


In the unbounded graph case (Theorem~\ref{thm:unbounded}), the input space contains infinitely many graphs; therefore, it is not compact under the discrete topology.
Hence, we cannot directly apply the Stone--Weierstrass theorem as in the bounded case. 
To obtain a stronger result, we have to \emph{complete} the set of all graphs, and prove the completed space is compact Hausdorff.
Since it is non-trivial to work directly with discrete graphs, we find that the graphon theory~\cite{lovasz2012large} fits our purpose.

\paragraph{Graphon} A sequence of graphs $G_1, G_2, \dots$ is a \emph{convergence} if the homomorphism density, $t(F, G_i)$, is a convergence for all simple graph $F$.
A limit of a convergence is called a \emph{graphon}, and the space obtained by adding the limits of the convergences is called the \emph{graphon space}, which is denoted by $\overline{\mathcal{G}}$.
See \cite{lovasz2012large} for the detail of this construction.
The following theorem is one of the most important results in graphon theory.

\begin{theorem}[Compactness Theorem~\cite{lovasz2012large,lovasz2006limits}]
\label{thm:graphon-compact}
The graphon space $\overline{\mathcal{G}}$ with the cut distance $\delta_\square$ is compact Hausdorff.
\end{theorem}

Now we can prove the graphon version of Theorem~\ref{thm:unbounded}.
\begin{theorem}
\label{thm:universality}
Any continuous $\mathcal{F}$-invariant function $f \colon \overline{\mathcal{G}} \to \mathbb{R}$ is arbitrary accurately approximated by a polynomial of $\{ t(F, \cdot) : F \in \mathcal{F} \}$.
\end{theorem}
\vspace{-1em}
\begin{proof}
The $\mathcal{F}$-indistinguishability forms a closed equivalence relation on $\overline{\mathcal{G}}$, where the homomorphism density is used instead of the homomorphism number.
Let $\overline{\mathcal{G}} / \mathcal{F}$ be the quotient space of this equivalence relation, which is compact Hausdorff in the quotient topology.

By the definition of the quotient topology, any continuous $\mathcal{F}$-invariant function is identified as a continuous function on $\overline{\mathcal{G}} / \mathcal{F}$.
Also, by the definition, the set of $\mathcal{F}$-homomorphisms separates the quotient space.
Therefore, the conditions of the Stone--Weierstrass theorem (Theorem~\ref{thm:stone-weierestrass}) are fulfilled.
\end{proof}
\vspace{-1em}

\subsection{Computational Complexity: Bounded Treewidth}
\label{sec:complexity}

Computing homomorphism numbers is, in general, an \#P-hard problem~\cite{diaz2002counting}.
However, if the pattern graph $F$ has bounded treewidth, homomorphism numbers can be computed in polynomial time.

A \emph{tree-decomposition}~\cite{robertson1986graph} of a graph $F$ is a tree $T = (V(T), E(T))$ with mapping $B \colon V(T) \to 2^{V(F)}$ such that (1) $\bigcup_{t \in V(T)} B(t) = V(F)$, (2) for any $(u, v) \in E(F)$ there exists $t \in V(T)$ such that $\{u, v\} \subseteq B(t)$, and (3) for any $u \in V(F)$ the set $\{ t \in V(T) : u \in B(t) \}$ is connected in $T$.
The \emph{treewidth} (abbreviated as ``tw'') of $F$ is the minimum of $\max_{t \in V(T)} |B(t) - 1|$ for all tree-decomposition $T$ of $F$.

\begin{theorem}[\cite{diaz2002counting}]
For any graphs $F$ and $G$, the homomorphism number $\mathrm{hom}(F, G)$ is computable in $O(|V(G)|^{\text{tw}(F)+1})$ time. 
\end{theorem}

The most useful case will be when $\mathcal{F}$ is the set of trees of size at most $k$.
The number of trees of size $k$ is a known integer sequence\footnote{\url{https://oeis.org/A000055}}. 
There are 106 non-isomorphic trees of size $k = 10$, which is computationally tractable in practice.
Also, in this case, the algorithm for computing $\mathrm{hom}(F, G)$ is easily implemented by dynamic programming with recursion as in Algorithm~\ref{alg:homtree}. 
This algorithm runs in $O(|V(G)| + |E(G)|)$ time.
For the non-featured case, we sets $x(u) = 1 \ \forall u \in V(G)$.
The simplicity of Algorithm~\ref{alg:homtree} comes from the fact that if $F$ is a tree then we only need to keep track of a vertex's immediate ancestor when we process that vertex by the \emph{visited} argument in the function \emph{recursion}.

\begin{algorithm}[tb]
   \caption{Compute $\mathrm{hom}(F, (G, x))$}
   \label{alg:homtree}
\begin{algorithmic}
   \STATE {\bfseries Input:} target graph $G$, pattern graph $F$, vertex features $x$
   \FUNCTION{recursion(current, visited)}
   \STATE hom\_x $\leftarrow x$ 
   \FOR{y {\bfseries in} $F$.neighbors(current)}
   \IF{y $\neq$ visited} 
   \STATE hom\_y $\leftarrow$ recursion(y, current)
   \STATE aux $\leftarrow$ [$\sum$ hom\_y[$G$.neighbors(i)] \textbf{for} i \textbf{in} $V(G)$]
   \STATE hom\_x $\leftarrow$ hom\_x $\ast$ aux \emph{(element-wise mult.)}
   \ENDIF
   \ENDFOR
   \STATE \textbf{return} hom\_x
   \ENDFUNCTION
   \STATE \textbf{Output:} $\sum$ recursion(0, -1)
\end{algorithmic}
\end{algorithm}
 
\section{Graphs with Features}
\label{sec:graph_with_features}
Biological/chemical datasets are often modeled as graphs with vertex features (attributed graphs).
In this section, we develop results for graphs with features.

\subsection{Definition}

A \emph{vertex-featured graph} is a pair $(G, x)$ of a graph $G$ and a function $x \colon V(G) \to \mathcal{X}$, where $\mathcal{X} = [0,1]^p$. 

Let $(G, x)$ be a vertex-featured graph.
For a finite set $U$ and a bijection $\sigma \colon V(G) \to U$, we denote by $x^\sigma$ the feature vector on $G^\sigma$ such that $x^\sigma(\sigma(u)) = x(u)$.
Two vertex-featured graphs $(G_1, x_1)$ and $(G_2, x_2)$ are \emph{isomorphic} if $G_1^\sigma = G_2$ and $x_1^\sigma = x_2$ for some $\sigma \colon V(G_1) \to V(G_2)$.

\subsection{Weighted Homomorphism Numbers}

We first consider the case that the features are non-negative real numbers.
Let $x(u)$ denote the feature of vertex $u$, \emph{weighted homomorphism number} is defined as follow:
\begin{align}
    \label{m:homlabel}
    \mathrm{hom}(F, (G, x)) = \sum_{\pi \in \mathrm{Hom}(F, G)} \prod_{u \in V(F)} x(\pi(u)),
\end{align}
and \emph{weighted homomorphism density} is defined by $t(F, (G, x)) = \mathrm{hom}(F, (G, x(u) / \sum_{v \in V(G)} x(v)))$.
This definition coincides with the homomorphism number and density if $x(u) = 1$ for all $u \in V(G)$.

The weighted version of the Lov\'{a}sz theorem holds as follows.
We say that two vertices $u, v \in V(G)$ are \emph{twins} if the neighborhood of $u$ and $v$ are the same.
The \emph{twin-reduction} is a procedure that iteratively selects twins $u$ and $v$, contract them to create new vertex $uv$, and assign $x(uv) = x(u) + x(v)$ as a new weight.
Note that the result of the process is independent of the order of contractions.
\begin{theorem}[\cite{freedman2007reflection}, \cite{cai2019theorem}]
\label{thm:freedman2017reflection}
Two graphs $(G_1, x_1)$ and $(G_2, x_2)$ are isomorphic after 
the twin-reduction and 
removing vertices of weight zero 
if and only if 
$\mathrm{hom}(F, (G_1, x_1)) = \mathrm{hom}(F, (G_2, x_2))$ for all simple graph $F$.
\end{theorem}

Now we can prove a generalization of the Lova\'{a}sz theorem.

\begin{theorem}
Two graphs $(G_1, x_1)$ and $(G_2, x_2)$ are isomorphic
if and only if 
$\mathrm{hom}(F, \phi, (G_1, x_1)) = \mathrm{hom}(F, \phi, (G_2, x_2))$ for all simple graph $F$ and some continuous function $\phi$.
\end{theorem}
\begin{proof}
It is trivial to see that if $(G_1, x_1)$ and $(G_2, x_2)$ are isomorphic then they produce the same homomorphism numbers.
Thus, we only have to prove the only-if part.

Suppose that the graphs are non-isomorphic.
By setting $\phi = 1$, we have the same setting as the feature-less case; hence, by Theorem~\ref{thm:lovasz1967operations}, we can detect the isomorphism classes of the underlying graphs.

Assuming $G_1$ and $G_2$ are isomorphic, we arrange the vertices of $V(G_1)$ in the increasing order of the features (compared with the lexicographical order).
Then, we arrange the vertices of $V(G_2)$ lexicographically smallest while the corresponding subgraphs induced by some first vertices are isomorphic.
Let us choose the first vertex $u \in V(G_1)$ whose feature is different to the feature of the corresponding vertex in $V(G_2)$.
Then, we define 
\begin{align*}
    \phi(z) = \begin{cases}
    1, & z \le_\text{lex} x_1(u), \\
    0, & \text{otherwise}, \end{cases}
\end{align*}
where $\le_\text{lex}$ stands for the lexicographical order.
Then, we have $\mathrm{hom}(F, \phi, (G_1, x_1)) \neq \mathrm{hom}(F, \phi, (G_2, x_2))$ as follows.
Suppose that the equality holds.
Then, by Theorem~\ref{thm:freedman2017reflection}, the subgraphs induced by vertices whose features are lexicographicallly smaller than or equal to $x_1(u)$ are isomorphic.
However, this contradicts the minimality of the ordering of $V(G_2)$.
Finally, by taking a continuous approximation of $\phi$, we obtain the theorem.
\end{proof}
%

\subsection{$(F,\phi)$-Homomorphism Number}

Let $(G, x)$ be a vertex-featured graph.
For a simple graph $F$ and a function $\phi \colon \mathbb{R}^p \to \mathbb{R}$, \emph{$(F, \phi)$-convolution} is given by
\begin{align}
    \mathrm{hom}(F, G, x; \phi) =
    \sum_{\pi \in \mathrm{Hom}(F, G)} \prod_{u \in V(F)} \phi(x(\pi(u))).
\end{align}
The $(F, \phi)$-convolution first transform the vertex features into real values by the encoding function $\phi$. 
Then this aggregates the values by the pattern graph $F$.
The aggregation part has some similarity with the convolution in CNNs.
Thus, we call this operation ``convolution.''

\begin{example}
Let $\circ$ be a singleton graph and $\phi$ be the $i$-th component of the argument.
Then, 
\begin{align}
\mathrm{hom}(F, G, x; \phi) 
= \sum_{u \in V(G)} x_i(u).
\end{align}
\end{example}
\begin{example}
Let $\begin{tikzpicture}
\node[draw,circle,inner sep=1.3pt] at (0,0) (a) {};
\node[draw,circle,inner sep=1.3pt] at (0.3,0) (b) {};
\draw[-] (a)--(b);\end{tikzpicture}$ be a graph of one edge and $\phi$ be the $i$-th component of the argument.
Then,
\begin{align}
\mathrm{hom}(\begin{tikzpicture}
\node[draw,circle,inner sep=1.3pt] at (0,0) (a) {};
\node[draw,circle,inner sep=1.3pt] at (0.3,0) (b) {};
\draw[-] (a)--(b);\end{tikzpicture}, G, x; \phi) 
= \sum_{(u,v) \in E(G)} x_i(u) x_i(v).
\end{align}
\end{example}
Algorithm~\ref{alg:homtree} implements this idea when $\phi$ is the identity function. 
We see that the $(F, \phi)$-convolution is invariant under graph isomorphism in the following result.
\begin{theorem}
\label{thm:invariance}
For a simple graph $F$, a function $\phi \colon \mathbb{R}^p \to \mathbb{R}$, a vertex-featured graph $(G, x)$, and a permutation $\sigma$ on $V(G)$, we have
\begin{align}
    \mathrm{hom}(F, G, x; \phi) = \mathrm{hom}(F, G^\sigma, x^\sigma, \phi).
\end{align}
\end{theorem}
\begin{proof}
$\mathrm{Hom}(\mathcal{F}, G^\sigma) = \{\sigma \circ \pi: \pi \in \mathrm{Hom}(\mathcal{F}, G)\}.$ 
Therefore, we have:
\begin{align}
    \mathrm{hom}(F, G^\sigma, x^\sigma; \phi) =
    \sum_{\pi \in \mathrm{Hom}(F, G)} \prod_{u \in V(F)} \phi(x^\sigma(\sigma \circ \pi(u))) \nonumber
\end{align}
From the definition, we have $x^\sigma(\sigma \circ \pi(u)) = x(\pi(u))$.
\end{proof}
Theorem~\ref{thm:invariance} indicates that for any $F$ and $\phi$, the $(F, \phi)$-convolution can be used as a feature map for graph classification problems.
To obtain a more powerful embedding, we can stack multiple $(F, \phi)$-convolutions.
Let $\mathcal{F}$ be a (possibly infinite) set of finite simple graphs and $\Phi$ be a (possibly infinite) set of functions.
Then \emph{$(\mathcal{F}, \Phi)$-convolution}, $\mathrm{hom}(\mathcal{F}, G, x; \Phi)$, is a (possibly infinite) vector:
\begin{align}
    \left[\, \mathrm{hom}(F, G, x; \phi) : F \in \mathcal{F}, \phi \in \Phi \,\right].
\end{align}
By Theorem~\ref{thm:invariance}, for any $\mathcal{F}$ and $\Phi$, the $(\mathcal{F}, \Phi)$-convolution is invariant under the isomorphism.
Hence, we propose to use $(\mathcal{F}, \Phi)$-convolution as a embedding of graphs.


\subsection{$(F,\phi)$-Homomorphism Number as Embedding}

Let $\Phi$ be a set of continuous functions.
As same as the featureless case, we propose to use the $(\mathcal{F}, \Phi)$-homomorphism numbers as an embedding.
We say that two featured graphs $(G_1, x_1)$ and $(G_2, x_2)$ are \emph{$(\mathcal{F}, \Phi)$-indistinguishable} if $\mathrm{hom}(F, \phi, (G_1, x_1)) = \mathrm{hom}(F, \phi, (G_2, x_2))$ for all $F \in \mathcal{F}$ and $\phi \in \Phi$.
A function $f$ is \emph{$(\mathcal{F}, \Phi)$-invariant} if $f(G_1, x_1) = f(G_2, x_2)$ for all $(\mathcal{F}, \Phi)$-indistinguishable $(G_1, x_1)$ and $(G_2, x_2)$.

\subsection{Universality Theorem}

The challenge in proving the universality theorem for the featured setting is similar to the featureless case, which is the difficulty of the topological space. 
We consider the quotient space of graphs with respect to $(\mathcal{F}, \Phi)$-indistinguishability.
Our goal is to prove this space is completed to a compact Hausdorff space.

With a slight abuse of notation, consider a function $\iota$ that maps a vertex featured graph $(G, x)$ to a $|\Phi|$-dimensional vector $[ (G, \phi(x) : \phi \in \Phi ] \in (\mathcal{G}/\mathcal{F})^\Phi$ where each coordinate is an equivalence class of $\mathcal{F}$-indistinguishable graphs.
This space has a bijection to the quotient space by $(\mathcal{F}, \Phi)$-indistinguishability.
Each coordinate of the $|\Phi|$-dimensional space is completed to a compact Hausdorff space~\cite{borgs2008convergent}.
Therefore, by the Tychonoff product theorem~\cite{hart2003encyclopedia}, the $|\Phi|$-dimensional space is compact.
The bijection between the quotient space shows the quotient space is completed by a compact Hausdorff space.
We denote this space by $\overline{\mathcal{G}}$. Under this space, we have the following result.
\begin{theorem}
Any continuous $(\mathcal{F}, \Phi)$-invariant function $\overline{\mathcal{G}} \to \mathbb{R}$ is arbitrary accurately approximated by a polynomial of $(G, x) \mapsto t(F, (G, \phi(x)))$.
\end{theorem}
\begin{proof}
The space $\overline{\mathcal{G}}$ is compact by construction.
The separability follows from the definition of $(\mathcal{F}, \Phi)$-invariant.
Therefore, by the Stone--Weierstrass theorem, we complete the proof.
\end{proof}

The strength of an embedding is characterized by the \emph{separability}.

\begin{lemma}
\label{lem:separability}
Let $\mathcal{F}$ be the set of all simple graphs and $\Phi$ be the set of all continuous functions from $[0,1]^p$ to $[0,1]$.
Then, $(G, x) \mapsto \mathrm{hom}(\mathcal{F}, G, x; \Phi)$ is injective.
\end{lemma}
\begin{proof}
Let $(G, x)$ and $(G', y)$ be two non-isomorphic vertex-featured graphs.
We distinguish these graphs by the homomorphism convolution.

If $G$ and $G'$ are non-isomorphic, 
by \cite{lovasz1967operations}, $\mathrm{hom}(\mathcal{F}, G, x; 1) \neq \mathrm{hom}(\mathcal{F}, G', y; 1)$ where $1$ is the function that takes one for any argument.

Now we consider the case that $G = G'$.
Let $\{1, \dots, n\}$ be the set of vertices of $G$.
Without loss of generality, we assume $x(1) \le x(2) \le \dots$ where $\le$ is the lexicographical order.
Now we find a permutation $\pi$ such that $G = G^\pi$ and $y(\pi(1)), y(\pi(2)), \dots$ are lexicographically smallest.
Let $u \in \{1, \dots, n\}$ be the smallest index such that $x(u) \neq y(u)$.
By the definition, $x(u) \le y(u)$.
We choose $\psi$ by 
\begin{align}
    \psi(x) = \begin{cases} 1, & x \le x(u), \\ 0, & \text{otherwise}. \end{cases}
\end{align}
Then, there exists $F \in \mathcal{F}$ such that $\mathrm{hom}(F, G, x; \psi) \neq \mathrm{hom}(F, G, y; \psi)$ because the graphs induced by $\{1, \dots, k\}$ and $\{\pi(1), \dots, \pi(k)\}$ are non-isomorphic because of the choice of $\pi$.

Now we approximate $\psi$ by a continuous function $\phi$.
Because $(F, \phi)$-convolution is continuous in the vertex weights (i.e., $\phi(x(u))$), by choosing $\phi$ sufficiently close to $\psi$, we get $\mathrm{hom}(F, G, x; \phi) \neq \mathrm{hom}(F, G, y; \phi)$.
\end{proof}

We say that a sequence $(G_i, x_i)$ \, ($i = 1, 2, \dots$) of featured graphs is an \emph{$(\mathcal{F}, \Phi)$-convergent} if for each $F \in \mathcal{F}$ and $\phi \in \Phi$ the sequence $\mathrm{hom}(F, G_i, x_i; \phi)$ \, ($i = 1, 2, \dots$) is a convergent in $\mathbb{R}$.
A function $f \colon (G, x) \mapsto f(G, x)$ is \emph{$(\mathcal{F}, \Phi)$-continuous} if for any $(\mathcal{F}, \Phi)$-convergent $(G_i, x_i)$ \, ($i = 1, 2, \dots$), the limit $\lim_{i \to \infty} f(G_i, x_i)$ of the function exists and its only depends on the limits $\lim_{i\to \infty} \mathrm{hom}(F, G_i, x_i, \phi)$ of the homomorphism convolutions for all $F \in \mathcal{F}$ and $\phi \in \Phi$.

Now we prove the universality theorem.
Let $\mathcal{H}$ be a dense subset of the set of continuous functions, e.g., the set of polynomials or the set of functions represented by a deep neural network.
We define $\mathcal{H}(\mathcal{G}; \mathcal{F}, \Phi)$ by 
\begin{align}
    &\mathcal{H}(\mathcal{G}; \mathcal{F}, \Phi) = \nonumber \\
    &\left\{ \sum_{F \in \mathcal{F}, \phi \in \Phi} h_{F, \phi}(\mathrm{hom}(F, \cdot; \phi) : h_{F, \phi} \in \mathcal{H} \right\}
\end{align}
where the argument of the function is restricted to $\mathcal{G}$.
This is the set of functions obtained by combining universal approximators in $\mathcal{H}$ and the homomorphism convolutions $\mathrm{hom}(F, G, x, \phi)$ for some $F \in \mathcal{F}$ and $\phi \in \Phi$.
Let $\mathcal{G}$ be a set of graphs, and let $C(\mathcal{G}; \mathcal{F}, \Phi)$ be the set of $(\mathcal{F}, \Phi)$-continuous functions defined on $\mathcal{G}$.
Then, we obtain the following theorem.
\begin{theorem}[Universal Approximation Theorem]
\label{thm:universality2}
Let $\mathcal{G}$ be a compact set of graphs whose number of vertices are bounded by a constant.
Then, $\mathcal{H}(\mathcal{G}; \mathcal{F}, \Phi)$ is dense in $C(\mathcal{G}; \mathcal{F}, \Phi)$.
\end{theorem}
\begin{proof}
Because the number of vertices are bounded, the space of converging sequences is identified as $\mathcal{G}$.
Therefore, this space is compact Hausdorff.
The separability is proved in Lemma~\ref{lem:separability}.
Hence, we can use the Stone--Weierstrass theorem to conclude this result.
\end{proof}

\subsection{The Choice of $\mathcal{F}$ and $\Phi$}

In an application, we have to choose $\mathcal{F}$ and $\Phi$ appropriately.
The criteria of choosing them will be the same as with non-featured case: Trade-off between \emph{representability} and \emph{efficiency}.
Representability requires that $(\mathcal{F}, \Phi)$-convolutions can separate the graphs in which we are interested in.
Efficiency requires that $(\mathcal{F}, \Phi)$-convolutions must be efficiently computable. This trivially limits both $\mathcal{F}$ and $\Phi$ as finite sets.

The choice of $\Phi$ will depend on the property of the vertex features.
We will include the constant function $1$ if the topology of the graph is important.
We will also include the $i$-th component of the arguments. 
If we know some interaction between the features is important, we can also include the cross terms.

The choice of $\mathcal{F}$ relates with the topology of the graphs of interest.
If $\Phi = \{ 1 \}$ where $1$ is the constant function, the homomorphism convolution coincides with the homomorphism number (Table~\ref{tab:F-indistinguishable}).

Here, we focus on the efficiency.
In general, computing $\mathrm{hom}(F, G, x, \phi)$ is \#P-hard. 
However, it is computable in polynomial time if $F$ has a bounded tree
The \emph{treewidth} of a graph $F$, denoted by $\mathrm{tw}(F)$, is a graph parameter that measures the tree-likeness of the graph. 
The following result holds.
\begin{theorem}
$\mathrm{hom}(F, G, x; \phi)$ is computable in $|V(G)|^{\text{tw}(F) + 1}$ time, $\mathrm{tw}(F)$ is the treewidth of $F$.
\end{theorem}




\section{Experimental results}\label{sec:exp}

\begin{table}[!ht]
  \vspace{-1em}
  \caption{Classification accuracy over 10 experiments}
  \label{tab:results}
   \begin{subtable}{.48\textwidth}
    \caption{Synthetic datasets}
   \resizebox{\textwidth}{!}{
    \begin{tabular}[C]{lccc}
\specialrule{.1em}{.05em}{.05em} 
\textsc{Methods} & CSL & BIPARTITE & PAULUS25 \\
\specialrule{.1em}{.05em}{.05em} 
\multicolumn{3}{l}{\emph{Practical models}} \\
\hline
GIN & 10.00 $\pm$ 0.00 & 55.75 $\pm$ 7.91 & 7.14 $\pm$ 0.00  \\
GNTK & 10.00 $\pm$ 0.00 & 58.03 $\pm$ 6.84 & 7.14 $\pm$ 0.00 \\
\specialrule{.1em}{.05em}{.05em} 
\multicolumn{3}{l}{\emph{Theory models}} \\
\hline
Ring-GNN & 10$\sim$80 $\pm$ 15.7 & 55.72 $\pm$ 6.95 & 7.15 $\pm$ 0.00\\
GHC-Tree   & 10.00 $\pm$ 0.00 & 52.68 $\pm$ 7.15 & 7.14 $\pm$ 0.00 \\
GHC-Cycle   & \textbf{100.0 $\pm$ 0.00} & \textbf{100.0 $\pm$ 0.00} & 7.14 $\pm$ 0.00 \\
\hline
\end{tabular}}
\vspace{1em}
  \end{subtable}
  \centering
  \begin{subtable}{.48\textwidth}
  \caption{Benchmark datasets}
    \resizebox{\textwidth}{!}{
    \begin{tabular}[C]{l c c c}
\specialrule{.1em}{.05em}{.05em} 
\textsc{Methods} & MUTAG & IMDB-BIN & IMDB-MUL \\
\specialrule{.1em}{.05em}{.05em} 
\multicolumn{3}{l}{\emph{Practical models}} \\
\hline
GNTK & 89.46 $\pm$ 7.03 & 75.61 $\pm$ 3.98 & 51.91 $\pm$ 3.56\\
GIN  & 89.40 $\pm$ 5.60 & 70.70 $\pm$ 1.10 & 43.20 $\pm$ 2.00 \\
PATCHY-SAN & 89.92 $\pm$ 4.50 & 71.00 $\pm$ 2.20 & 45.20 $\pm$ 2.80 \\
WL kernel & 90.40 $\pm$ 5.70 & 73.80 $\pm$ 3.90 & 50.90 $\pm$ 3.80\\
\specialrule{.1em}{.05em}{.05em} 
\multicolumn{3}{l}{\emph{Theory models}} \\
\hline
Ring-GNN & 78.07 $\pm$ 5.61 & 73.00 $\pm$ 5.40 & 48.20 $\pm$ 2.70\\
GHC-Tree  & 89.28 $\pm$ 8.26 & 72.10 $\pm$ 2.62 & 48.60 $\pm$ 4.40 \\
GHC-Cycles & 87.81 $\pm$ 7.46 & 70.93 $\pm$ 4.54 & 47.41 $\pm$ 3.67 \\
\specialrule{.1em}{.05em}{.05em} 
\end{tabular}}
\end{subtable}
\end{table}

\subsection{Classification models}\label{subsec:clf}

The realization of our ideas in Section~\ref{sec:graph_no_features} and Section~\ref{sec:graph_with_features} are called Graph Homomorphism Convolution (GHC-*) models (due to their resemblance to the $\mathcal{R}-$convolution~\cite{haussler1999convolution}). Here, we give specific formulations for two practical embedding maps: GHC-Tree and GHC-Cycle. These embedding maps are then used to train a classifier (Support Vector Machine). We report the 10-folds cross-validation accuracy scores and standard deviations in Table~\ref{tab:results}.

\paragraph{GHC-Tree} We let $\mathcal{F}_\text{tree(6)}$ to be all simple trees of size at most 6. 
Algorithm~\ref{alg:homtree} implements Equation~\ref{m:homlabel} for this case. 
Given $G$ and vertex features $x$, the $i$-th dimension of the embedding vector is 
\begin{align*}
    \text{GHC-Tree}(G)_i = \mathrm{hom}(\mathcal{F}_\text{tree(6)}[i], (G,x)).
\end{align*}

\paragraph{GHC-Cycle} We let $\mathcal{F}_\text{cycle(8)}$ to be all simple cycles of size at most 8. This variant of GHC cannot distinguish iso-spectral graphs. The $i$-th dimensional of the embedding vector is
\begin{align*}
    \text{GHC-Cycle}(G)_i = \mathrm{hom}(\mathcal{F}_\text{cycle(8)}[i], G).
\end{align*}

With this configuration, GHC-Tree($G$) has 13 dimensions and GHC-Cycle($G$) has 7 dimensions.

\paragraph{Other methods} To compare our performance with other approaches, we selected some representative methods. GIN~\cite{gin} and PATCHY-SAN~\cite{niepert2016learning} are representative of neural-based methods.
WL-kernel~\cite{shervashidze2011weisfeiler} is a widely used efficient method for graph classifications.
GNTK~\cite{gntk} is a recent neural tangent approach to graph classification. 
We also include results for Ring-GNN~\cite{chen2019equivalence} as this recent model used in theoretical studies performed well in the Circular Skip Links synthetic dataset~\cite{murphy19pooling}.
Except for setting the number of epochs for GIN to be 50, we use the default hyperparameters provided by the original papers. 
More details for hyperparamters tuning and source code is available in the Supplementary Materials. 

\subsection{Synthetic Experiments}\label{subsec:syn}

\paragraph{Bipartite classification} We generate a binary classification problem consisting of 200 graphs, half of which are random bipartite graphs with density $p=0.2$ and the other half are Erd\H{o}s-R\'{e}nyi graphs with density $p=0.1$. These graphs have from 40 to 100 vertices. According to Table~\ref{tab:F-indistinguishable}, GHC-Cycle should work well in this case while GHC-Tree can not learn which graph is bipartite. More interestingly, as shown in Table~\ref{tab:results}, other practical models also can not work with this simple classification problem due to their capability limitation (1-WL).

\paragraph{Circular Skip Links} We adapt the synthetic dataset used by \cite{murphy19pooling} and \cite{chen2019equivalence} to demonstrate another case where GIN, Relational Pooling~\cite{murphy19pooling}, and Order 2 G-invariant~\cite{maron2018invariant} do not perform well.
Circular Skip Links (CSL) graphs are undirected regular graphs with the same degree sequence (4's).
Since these graphs are not cospectral, GHC-Cycle can easily learn them with 100\% accuracy.
\citeauthor{chen2019equivalence} mentioned that the performance of GNN models could vary due to randomness (accuracies ranging from 10\% to 80\%). However, it is not the case for GHC-Cycle.
CSL classification results shows another benefit of using $F$ patterns as an inductive bias to implement a strong classifier without the need of additional features like Ring-GNN-SVD~\cite{chen2019equivalence}.

\paragraph{Paulus graphs} We prepare 14 non-isomorphic cospectral strongly regular graphs known as the Paulus graphs\footnote{https://www.distanceregular.org/graphs/paulus25.html} and create a dataset of 210 graphs belonging to 14 isomorphic groups.
This is a hard example because these graphs have exactly the same degree sequence and spectrum. 
In our experiments, no method achieves accuracy higher than random guesses (7.14\%). 
This is a case when exact isomorphism tests clearly outperform learning-based methods.
In our experiments, homomorphisms up to graph index 100 of NetworkX's graph atlas still fail to distinguish these isomorphic classes.
We believe further studies of this case could be fruitful to understand and improve graph learning.

\subsection{Benchmark Experiments}

We select 3 datasets from the TU Dortmund data collection~\cite{tud}: MUTAG dataset~\cite{debnath1991structure}, IMDB-BINARY, and IMDB-MULTI~\cite{yanardag2015deep}. 
These datasets represent with and without vertex features graph classification settings. 
We run and record the 10-folds cross-validation score for each experiment. 
We report the average accuracy and standard deviation of 10 experiments in Table~\ref{tab:results}. 
More experiments on other datasets in the TU Dortmund data collection, as well as the details of each dataset, are provided in the Appendix. 

\subsection{Running time}

Although homomorphism counting is \#P-complete in general, polynomial and linear time algorithms exist under the bounded tree-width condition~\cite{diaz2002counting}. 
Figure~\ref{fig:runtime} shows that our method runs much faster than other practical models.
The results are recorded from averaging total runtime in seconds for 10 experiments, each computes the 10-folds cross-validation accuracy score.
In principle, GHC can be linearly distributed to multiple processes to further reduce the computational time making it an ideal baseline model for future studies.

\begin{figure}
\begin{center}
   \includegraphics[width=\linewidth]{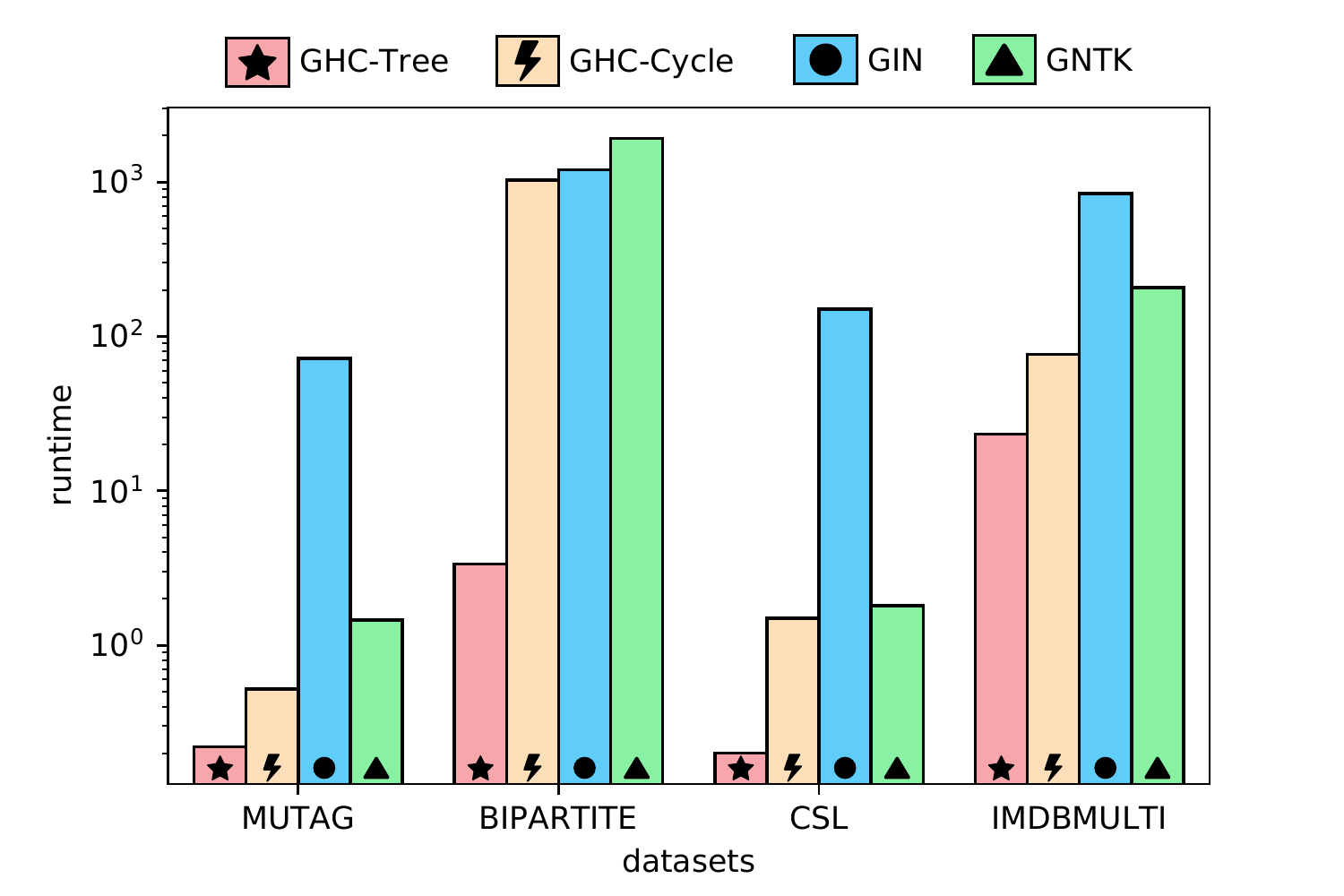}
\end{center}
   \caption{Runtime (sec) in log-scale for one 10-folds run}
\label{fig:runtime}
\end{figure}

\section{Conclusion}\label{sec:conclusion}
In this work we contribute an alternative approach to the question of quantifying a graph classification model's capability beyond the tensorization order and the Weisfeiler-Lehman isomorphism test. 
In principle, tensorized graph neural networks can implement homomorphism numbers, hence our work is in coherence with prior work.
However, we find that the homomorphism from $\mathcal{F}$ to $G$ is a more ``fine-grained'' tool to analyze graph classification problems as studying $\mathcal{F}$ would be more intuitive (and graph-specific) than studying the tensorization order. Since GHC is a more restricted embedding compared to tensorized graph neural networks such as the model proposed by~\cite{keriven2019universal}, the universality result of GHC can be translated to the universality result of any other model that has the capability to implement the homomorphism numbers.

Another note is about the proof for Theorem~\ref{thm:unbounded} (universality on unbounded graphs). 
In order to prove this result, we made an assumption about the topology of $f$ and also about the graph of interest belongs to the graphon space. 
While the graphon space is natural in our application to prove the universality, there are a few concerns.
First, we assumed that the graphons exist for graphs of interest.
However, it might not be true in general.
Second, graph limit theory is well-studied in dense graphs while sparse graph problems remain largely open.




\section*{Acknowledgement} 
HN is supported by the Japanese Government Scholarship (Monbukagakusho: MEXT SGU 205144). We would like to thank the anonymous reviewers whose comments have helped improving the quality of our manuscript. We thank Jean-Philippe Vert for the helpful discussion and suggestions on our early results. We thank Matthias Springer, Sunil Kumar Maurya, and Jonathon Tanks for proof-reading our manuscripts.

\bibliography{ghp}
\bibliographystyle{icml2020}
\section*{APPENDIX}
\label{sec:appendix}
\renewcommand{\thesection}{\Alph{section}}
\setcounter{section}{0}

We present additional information for the theoretical framework (Section~\ref{sec:graph_no_features} and Section~\ref{sec:graph_with_features}) and experimental settings (Section~\ref{sec:exp}) here. The Appendix is organized as follow:
\begin{itemize}
    \setlength{\itemsep}{0pt}
    \item Section~\ref{sec:a_exp} gives details of the configurations for GHC-$\ast$ and other GNNs.
    \item Section~\ref{sec:a_data} gives details of synthetic and real-world datasets used in this paper. We also provide some additional results on other real-world datasets.
\end{itemize}

\section{Implementation Details}
\label{sec:a_exp}

The source code for GHC is provided with this supplementary document.
The main implementation is in file \texttt{homomorphism.py}.
Aside from Algorithm~\ref{alg:homtree}, which can be implemented directly with numpy and run with networkx, other types of homomorphism counting are implemented with \texttt{C++} and called from \texttt{homomorphism.py}.
The implementation for general homomorphism is called \texttt{homlib}.
We include an instruction to install \texttt{homlib} in its \texttt{README.md}. 
All our experiments are run on a PC with the following specifications. Kernel: \texttt{5.3.11-arch1-1}; CPU: Intel \texttt{i7-8700K (12) \@ 4.7 GHz}; GPU: NVIDIA GEFORCE GTX 1080 Ti 11GB; Memory: 64 GB. Note that GPU is only used for training GIN~\cite{gin}.

\paragraph{Benchmark Experiments} The main file to run experiment on real-world (benchmark) datasets is \texttt{tud.py}.
This is a simple classification problem where each graph in a dataset belongs to a single class.
While any other classifier can be used with GHC, we provide the implementation only for Support Vector Machines (\texttt{scikit-learn}) with
One-Versus-All multi-class algorithm.
We preprocess the data using the \texttt{StandardScaler} provided with \texttt{scikit-learn}. 
As described in the main part of this paper, we report the best 10-folds cross validation accuracy scores across different SVM configurations.
The parameter settings for this experiments are:
\begin{itemize}
    \setlength{\itemsep}{0pt}
    \item Homomorphism types: Tree, LabelTree (weighted homomorphism), and Cycle.
    \item Homomorphism size: 6 for trees and 8 for cycles. Theoretically, the homomorphism size increase implies performance increase, but in practice we observe no improvement in classification accuracy beyond size 6. The number of non-isomorphic trees of size $k$ is presented in Table~\ref{tab:trees-of-size-k}.
    \item SVC kernel: Radial Basis Function, Polynomial (max degree = 3).
    \item SVC regularization parameter (C): $20$ values in the log-space from $10^{-2}$ to $10^{5}$.
    \item SVC kernel coefficient (gamma): `scale' (\texttt{1 / (n\_features * X.var()})
\end{itemize}

\begin{table}[tb]
    \caption{The number of non-isomorphic trees of size $k$.\footnotemark}
    \label{tab:trees-of-size-k}
    \centering
    \begin{tabular}{c|ccccccc}
    \toprule
    $k$      & 2 & 3 & 4 & 5 & 6 & 7  & 8  \\
    \# trees & 1 & 1 & 2 & 3 & 6 & 11 & 23 \\ \midrule 
    $k$      & 9  & 10  & 11  & 12  & 13 & 14  & 15 \\
    \# trees & 47 & 106 & 235 & 551 & 1301 & 3159 & 7741 \\ 
    \bottomrule
    \end{tabular}
\end{table}

\footnotetext{\url{https://oeis.org/A000055}}
For other GNNs, we use the default hyperparameters used by the original paper.
To run GIN~\cite{gin}, we fix the number of epochs at 100 and enable the \texttt{use degree as tags} by default for all datasets.
We limit the number of threads used by GNTK~\cite{gntk} to 8.

\paragraph{Synthetic Experiments} The main file to run experiment on real-world (benchmark) datasets is \texttt{synthetic.py} and the implementation of synthetic datasets can be found in \texttt{utils.py}.
Since these experiments focus on the capability of GHC, we can achieve the best performance with just a simple classifier.
The parameter settings for this experiments are:
\begin{itemize}
    \setlength{\itemsep}{0pt}
    \item Homomorphism types: Tree and Cycle.
    \item Homomorphism size: 6 for trees and 8 for cycles.
    \item SVC kernel: Radial Basis Function.
    \item SVC regularization parameter (C): Fix at $1.0$.
    \item SVC kernel coefficient (gamma): Fix at $1.0$.
\end{itemize}

We provide in our source code helper functions which behave the same as the default dataloader used by GIN and GNTK implementations.
The external loaders are provided in \texttt{externals.py}.
Users can copy-paste (or import) these loader into the repository provided by GIN and GNTK to run our synthetic experiments.
The settings for other models are set as in the benchmark experiments.

\paragraph{Timing Experiments} We measure run-time using the \texttt{time} module provided with Python 3.7. 
The reported time in Figure~\ref{fig:runtime} is the total run-time (in seconds) including homomorphism time and prediction time for our model as well as kernel learning time and prediction time for others.

\section{Datasets}
\label{sec:a_data}

\begin{figure}
\begin{center}
   \includegraphics[width=\linewidth]{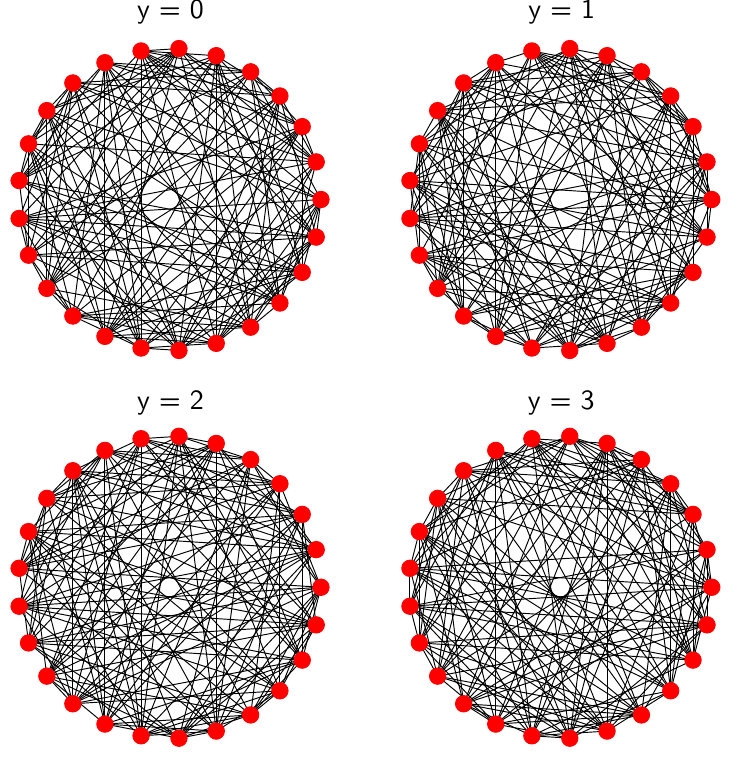}
\end{center}
   \caption{Four first isomorphic groups in Paulus25 dataset}
\label{fig:paulus}
\end{figure}

As other works in the literature, we use the TU Dortmund data collections~\cite{tud}.
The overview of these datasets are provided in Table~\ref{tab:datainfo}. 
We provide additional classification results for these datasets in Table~\ref{tab:results_lit_nonf_extra} and Table~\ref{tab:results_lit_f_extra}. 
We also provide some example of MUTAG data in Figure~\ref{fig:mutag}, four first isomorphic groups of PAULUS25 in Figure~\ref{fig:paulus}, plot of trees and cycles used in GHC-Tree (Figure~\ref{fig:trees6} and GHC-Cycles (Figure~\ref{fig:cycles8}).

\begin{figure*}
\begin{center}
   \includegraphics[width=\linewidth]{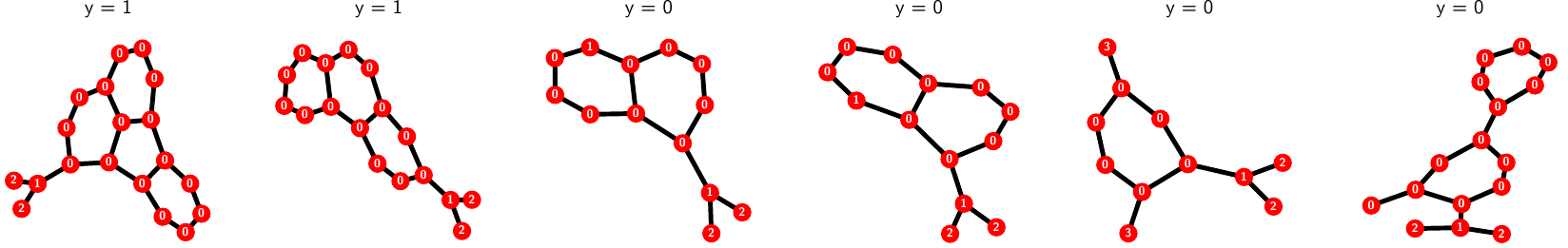}
\end{center}
   \caption{Example of MUTAG data}
\label{fig:mutag}
\end{figure*}

\begin{figure*}
\begin{center}
   \includegraphics[width=\linewidth]{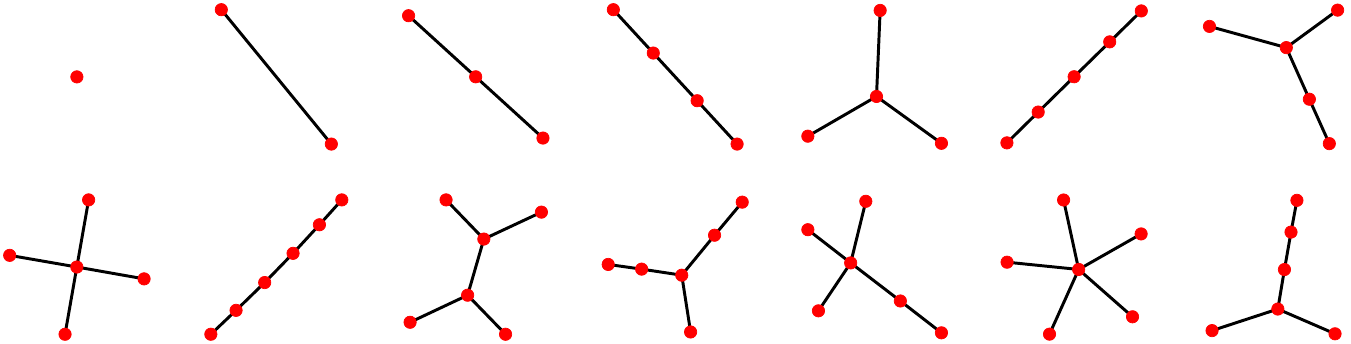}
\end{center}
   \caption{Elements of $\mathcal{F}_\mathrm{tree(6)}$}
\label{fig:trees6}
\end{figure*}

\begin{table}
\begin{center}
\begin{tabular}{l|rrrrr}
\textsc{Datasets} & $N$ & $\overline{n}$ & $\vert c \vert$ & $\mathcal{X}$ & $\mathcal{T}$\\
\midrule
MUTAG & 188 & 17.9 & 2 & no & yes\\
PTC-MR & 344 & 25.5 & 2 & no & yes\\
NCI1 & 4110 & 29.8 & 2 & no & yes\\
PROTEINS & 1113 & 39.1 & 2 & yes & yes\\
D\&D & 1178 & 284.3 & 2 & no & yes\\
BZR & 405 & 35.7 & 2 & yes & yes\\
RDT-BIN & 2000 & 429.6 & 2 & no & no\\
RDT-5K & 5000 & 508.5 & 5 & no & no\\ 
RDT-12K & 11929 & 391.4 & 11 & no & no\\
COLLAB & 5000 & 74.5 & 3 & no & no\\ 
IMDB-BIN & 1000 & 19.8 & 2 & no & no\\ 
IMDB-MUL & 1500 & 13.0 & 3 & no & no\\
Bipartite & 200 & 70.0 & 2 & no & no\\ 
CSL & 150 & 41.0 & 10 & no & no\\ 
Paulus 25 & 210 & 25.0 & 14 & no & no\\
\end{tabular}
\end{center}
\caption{Overview of the datasets in this paper. Here, $N$ denotes total number of graphs, $\overline{n}$ denotes the average number of nodes, $\vert c \vert$  denotes number of classes, $\mathcal{X}$ denotes if the dataset consists of vertex features, and $\mathcal{T}$ denotes if the dataset consists of vertex tags (or types).}
\label{tab:datainfo}
\end{table}

\begin{figure}
\begin{center}
   \includegraphics[width=\linewidth]{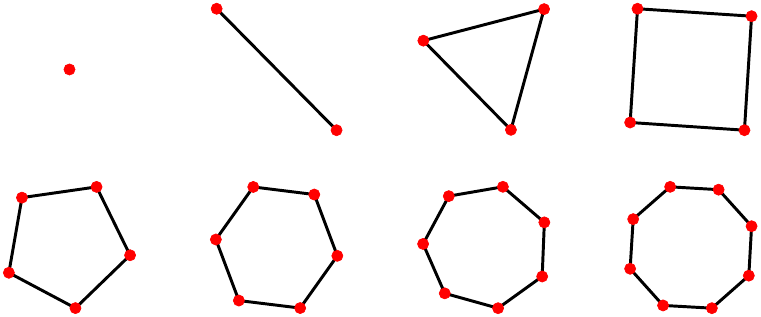}
\end{center}
   \caption{Elements of $\mathcal{F}_\mathrm{cycles(8)}$}
\label{fig:cycles8}
\end{figure}

\begin{table*}
\begin{center}
\begin{tabular}{|l|c|c|c|c|c|c|c|c|}
\hline
\multirow{2}{*}{\hspace{1em}\textsc{Methods}} & \multicolumn{6}{c|}{\textsc{Datasets}} \\\cline{2-7}
 & RDT-BIN & RDT-5K & RDT-12K & COLLAB & IMDB-BIN & IMDB-MUL \\
\hline
\multicolumn{5}{l}{\emph{Our experiments} (Average over 10 runs of stratified 10-folds CV)} \\
\hline
GHC-Tree   & 88.42 $\pm$ 2.05 & 52.98 $\pm$ 1.83 & 44.8 $\pm$ 1.00 & 75.23 $\pm$ 1.71 & 72.10 $\pm$ 2.62 & 48.60 $\pm$ 4.40 \\
GHC-Cycles & 87.61 $\pm$ 2.45 & 52.45 $\pm$ 1.24 & 40.9 $\pm$ 2.01 & 72.59 $\pm$ 2.02 & 70.93 $\pm$ 4.54 & 47.61 $\pm$ 3.67 \\
\hline
GIN & 74.10 $\pm$ 2.34 & 46.74 $\pm$ 3.07 & 32.56 $\pm$ 5.33 & 75.90 $\pm$ 0.81 & 70.70 $\pm$ 1.1 & 43.20 $\pm$ 2.00 \\
GNTK & - & - & - & 83.70 $\pm$ 1.00 & 75.61 $\pm$ 3.98 & 51.91 $\pm$ 3.56\\
\hline
\multicolumn{5}{l}{\emph{Literature} (One run of stratified 10-folds CV)} \\
\hline
GIN & 92.4 $\pm$ 2.5 & 57.5 $\pm$ 1.5 & - & 80.2 $\pm$ 1.9 & 75.1 $\pm$ 5.1 & 52.3 $\pm$ 2.8 \\
PATCHY-SAN & 86.3 $\pm$ 1.6 & 49.1 $\pm$ 0.7 & - & 72.6 $\pm$ 2.2 & 71.0 $\pm$ 2.2 & 45.2 $\pm$ 2.8 \\
WL kernel & 80.8 $\pm$ 0.4 & - & - & 79.1 $\pm$ 0.1 & 73.12 $\pm$ 0.4 & -\\
Graphlet kernel & 60.1 $\pm$ 0.2 & - & 31.8 & 64.7 $\pm$ 0.1 & - & -\\
AWL kernel & 87.9 $\pm$ 2.5 & 54.7 $\pm$ 2.9 & - & 73.9 $\pm$ 1.9 & 74.5 $\pm$ 5.9 & 51.5 $\pm$ 3.6 \\
WL-OA kernel & 89.3 & - & - & 80.7 $\pm$ 0.1 & - & -\\
WL-W kernel & - & - & - & - & 74.37 $\pm$ 0.83 & - \\
GNTK & - & - & - & 83.6 $\pm$ 1.0 & 76.9 $\pm$ 3.6 & 52.8 $\pm$ 4.6 \\
\hline
\end{tabular}
\end{center}
\caption{Graph classification accuracy (percentage) on popular non-vertex-featured benchmark datasets. This table provides the results obtained by averaging 10 times the 10-folds cross-validation procedure. Note that the results reported in the literature are run for only one 10-folds cross-validation. ``-'' denotes the result is not available or the experiment runs for more than 2 days (48 hours).}
\label{tab:results_lit_nonf_extra}
\end{table*}

\begin{table*}
\begin{center}
\begin{tabular}{|l|c|c|c|c|c|c|c|c|}
\hline
\multirow{2}{*}{\hspace{1.6em}\textsc{Methods}} & \multicolumn{6}{c|}{\textsc{Datasets}} \\\cline{2-7}
 & MUTAG & PTC-MR & NCI1 & PROTEINS & D\&D & BZR \\
\hline
\multicolumn{5}{l}{\emph{Our experiments} (Average over 10 runs of stratified 10-folds CV)} \\
\hline
GHC-Tree         & 89.28 $\pm$ 8.26 & 52.98 $\pm$ 1.83 & 48.8 $\pm$ 1.00 & 75.23 $\pm$ 1.71 & 72.10 $\pm$ 2.62 & 48.60 $\pm$ 4.40 \\
GHC-Cycle        & 87.81 $\pm$ 7.46 & 50.97 $\pm$ 2.13 & 47.4 $\pm$ 1.02 & 74.30 $\pm$ 1.93 & 70.10 $\pm$ 2.49 & 47.20 $\pm$ 3.84 \\
GHC-LabelTree  & 88.86 $\pm$ 4.82 & 59.68 $\pm$ 7.98 & 73.95 $\pm$ 1.99 & 73.27 $\pm$ 4.17 & 76.50 $\pm$ 3.15 & 82.82 $\pm$ 4.37 \\
\hline
GIN & 74.10 $\pm$ 2.34 & 46.74 $\pm$ 3.07 & 76.67 $\pm$ 1.16 & 75.9 $\pm$ 0.81 & 70.70 $\pm$ 1.1 & 43.20 $\pm$ 2.00 \\
GNTK & 89.65 $\pm$ 7.5 & 68.2 $\pm$ 5.8 & 85.0 $\pm$ 1.2 & 76.60 $\pm$ 5.02 & 75.61 $\pm$ 3.98 & 83.64 $\pm$ 2.95\\
\hline
\multicolumn{5}{l}{\emph{Literature} (One run of stratified 10-folds CV)} \\
\hline
GIN & 89.4 $\pm$ 5.6 & 64.6 $\pm$ 7.0 & 82.7 $\pm$ 1.7 & 76.2 $\pm$ 2.8 & - & -\\
PATCHY-SAN & 92.5 $\pm$ 4.2 & 60.0 $\pm$ 4.8 & 78.6 $\pm$ 1.9 & 75.9 $\pm$ 2.8 & 77.12 $\pm$ 2.41 & - \\
WL kernel & 90.4 $\pm$ 5.7 & 59.9 $\pm$ 4.3 & 86.0 $\pm$ 1.8 & 75.0 $\pm$ 3.1 & 79.78 $\pm$ 0.36 & 78.59 $\pm$ 0.63\\
Graphlet kernel & 85.2 $\pm$ 0.9 & 54.7 $\pm$ 2.0 & 70.5 $\pm$ 0.2 & 72.7 $\pm$ 0.6 & 79.7 $\pm$ 0.7 & -\\
AWL kernel & 87.9 $\pm$ 9.8 & - & - & - & - & -\\
WL-OA kernel & 84.5 $\pm$ 0.17 & 63.6 $\pm$ 1.5 & 86.1 $\pm$ 0.2 & 76.4 $\pm$ 0.4 & 79.2 $\pm$ 0.4 & -\\
WL-W kernel & 87.27 $\pm$ 1.5 & 66.31 $\pm$ 1.21 & 85.75 $\pm$ 0.25 & 77.91 $\pm$ 0.8 & 79.69 $\pm$ 0.50 & 84.42 $\pm$ 2.03 \\
GNTK & 90.00 $\pm$ 8.5 & 67.9 $\pm$ 6.9 & 84.2 $\pm$ 1.5 & 75.6 $\pm$ 4.2 & - & - \\
\hline
\end{tabular}
\end{center}
\caption{Graph classification accuracy (percentage) on popular vertex-featured (vertex-labeled) benchmark datasets. This table provides the results obtained by averaging 10 times the 10-folds cross-validation procedure. ``-'' denotes the result is not available in the literature or the experiment runs for more than 2 days (48 hours).}
\label{tab:results_lit_f_extra}
\end{table*}

\begin{remark}
We conjecture that the above result can be extended to the infinite graphs (say, graphons).
\end{remark}

\end{document}